\documentclass[letterpaper, 10 pt, conference]{ieeeconf}
\IEEEoverridecommandlockouts
\let\labelindent\relax
\usepackage{enumitem}
\usepackage{balance}
\usepackage[font={small}]{caption}
\usepackage{subcaption}
\usepackage{array}
\usepackage{textcomp}
\usepackage{mathtools, nccmath}
\usepackage{graphicx}
\usepackage{amsfonts}
\usepackage{amsmath}
\usepackage{amssymb}
\usepackage{algorithm}
\usepackage{algorithmic}
\usepackage[absolute,overlay]{textpos}
\usepackage{hyperref}
\usepackage{tikz}
\usepackage{arydshln}
\usepackage{multirow}
\usepackage{bm}
\usepackage{epstopdf}
\usepackage{cite}
\usepackage{siunitx}
\usepackage{multirow}
\usepackage{lipsum}

\usepackage{amsthm}

\newtheorem{theorem}{Theorem}
\newtheorem{lemma}{Lemma}

\newtheorem{remark}{Remark}
\theoremstyle{definition}
\newtheorem{definition}{Definition}
\newtheorem{problem}{Problem}

\title{\LARGE \bf
Learning Safety for Obstacle Avoidance via Control Barrier Functions}

\author{Shuo Liu$^{1}$, Zhe Huang$^{1}$ and Calin A. Belta$^{2}$
\thanks{This work was supported in part by the NSF under grant IIS-2024606 at Boston University.}
\thanks{$^{1}$S. Liu and Z. Huang are with Boston
University, Brookline, MA, USA {\tt\small \{liushuo, huangz7\}@bu.edu}}
\thanks{$^{2}$C. Belta is with the Department of Electrical and Computer Engineering and with the Department of Computer Science, University of Maryland, College Park, MD, USA 
        {\tt\small cbelta@umd.edu}}%
}

\begin{document} 
\maketitle

\begin{abstract}
Obstacle avoidance is central to safe navigation, especially for robots with arbitrary and nonconvex geometries operating in cluttered environments. Existing Control Barrier Function (CBF) approaches often rely on analytic clearance computations, which are infeasible for complex geometries, or on polytopic approximations, which become intractable when robot configurations are unknown. To address these limitations, this paper trains a residual neural network on a large dataset of robot–obstacle configurations to enable fast and tractable clearance prediction, even at unseen configurations. The predicted clearance defines the radius of a Local Safety Ball (LSB), which ensures continuous-time collision-free navigation. The LSB boundary is encoded as a Discrete-Time High-Order CBF (DHOCBF), whose constraints are incorporated into a nonlinear optimization framework. To improve feasibility, a novel relaxation technique is applied. The resulting framework ensure that the robot’s rigid-body motion between consecutive time steps remains collision-free, effectively bridging discrete-time control and continuous-time safety. We show that the proposed method handles arbitrary, including nonconvex, robot geometries and generates collision-free, dynamically feasible trajectories in cluttered environments. Experiments demonstrate millisecond-level solve times and high prediction accuracy, highlighting both safety and efficiency beyond existing CBF-based methods.
\end{abstract}

\section{Introduction}
\label{sec:Introduction}
Ensuring safety in autonomous robotic navigation is a fundamental yet challenging problem, especially when robots and obstacles possess complex (e.g., nonconvex, nonsmooth) geometries. A widely used formal tool for safety-critical control is the Control Barrier Function (CBF), which enforces safety by rendering a designated set forward invariant for a dynamical system. Recent studies have shown that CBFs can be combined with Control Lyapunov Functions (CLFs) to achieve both stabilization and safety while respecting input constraints. This integration yields real-time implementable optimization problems, typically formulated as Quadratic Programs (QPs) when the cost is quadratic \cite{ames2014control, ames2016control}.

Beyond relative-degree one constraints, several extensions of CBFs have been proposed. Exponential CBFs \cite{nguyen2016exponential} and High-Order CBFs (HOCBFs) \cite{xiao2021high, tan2021high} address safety functions with higher relative degrees. These methods have been applied to robotic obstacle avoidance (e.g., \cite{thirugnanam2022duality, chen2025control, wu2025optimization,peng2023safe}). However, due to the complexity of obstacle shapes, the authors of \cite{peng2023safe} employed logistic regression to construct polynomial barrier functions that approximate obstacle boundaries, enabling analytical enforcement of nonnegative clearance. Nevertheless, this method could not handle robots with complex geometries. In \cite{thirugnanam2022duality, chen2025control}, the authors focused on polytopic robots and obstacles, where the CBFs were defined through polytope–to-polytope distances. Such functions were not globally differentiable, and hence safety could not be guaranteed at configurations where the CBF derivative was undefined. Although the authors of \cite{wu2025optimization} proposed a smooth approximation to unify multiple CBFs into a single differentiable function, this led to overly conservative control, since smoothing enlarged the safe set margin and restricted feasible control actions more than necessary.

To avoid reliance on differentiable CBFs and address discrete-time systems, Discrete-Time CBFs (DCBFs) were introduced in \cite{agrawal2017discrete}, and later extended to Discrete-Time High-Order CBFs (DHOCBFs) \cite{xiong2022discrete,liu2023iterative}. These methods have been applied to collision avoidance between polytopic robots and polytopic obstacles \cite{thirugnanam2022safety}, and even extended to handle robots navigating around nonconvex obstacles \cite{liu2025safety, liu2024learning}. However, \cite{thirugnanam2022safety} remains limited to polytopic robots. Accurately approximating complex geometries with polytopes incurs significant computational overhead at each robot configuration. Moreover, constructing such approximations at unknown future configurations is intractable, fundamentally limiting the applicability of this approach in predictive control settings. While the methods proposed in \cite{liu2025safety, liu2024learning} effectively address nonconvex obstacle avoidance, they do not extend to handling robots with complex shapes.


More recently, learning-based methods that directly regress a Signed Distance Function (SDF) have been explored (e.g., \cite{liu2023collision, li2024representing, das2025rnbf, long2021learning}). While such approaches can predict the clearance between the robot and obstacles, the methods in \cite{liu2023collision, li2024representing} enforce safety only at discrete configurations, without guaranteeing collision avoidance along the continuous trajectory between these configurations. In contrast, \cite{das2025rnbf, long2021learning} achieve continuous-time obstacle avoidance by incorporating SDF-based CBF constraints, but their frameworks are limited to point robots or robots with circular shapes, and are not applicable to robots with complex geometries.

These limitations motivate our proposed approach, which combines clearance prediction, local safety region construction, and continuous-time safety enforcement between discrete-time steps into a unified framework. A residual neural network is trained on a large dataset of robot–obstacle configurations to predict the clearance. The predicted clearance defines a Local Safety Ball (LSB), which is anchored to the robot's boundary to construct a local safety region and ensure continuous-time safety. The LSB is then encoded as DHOCBF constraints, which are incorporated into a nonlinear optimization problem. To improve feasibility, a novel relaxation technique is applied. This overall framework guarantees that the robot’s rigid-body motion between discrete time steps remains collision-free, thereby bridging discrete-time control with continuous-time collision avoidance.

  
    

We demonstrate through numerical results that the proposed framework can generate dynamically feasible, collision-free trajectories for robots with various shapes navigating cluttered, maze-like environments. The results further show that each optimization problem can be solved within milliseconds, confirming the approach’s efficiency and scalability.

\begin{table}[]
\vspace*{1mm}
\centering
\caption{List of acronyms.}
\begin{tabular}{|c|c|}
\hline
Acronyms & Meaning \\ \hline
CBF & Control Barrier Function \\
 DCBF & Discrete-Time Control Barrier Function \\
 DHOCBF & Discrete-Time High-Order Control Barrier Function \\
 SDF &  Signed Distance Function\\
 LSB & Local Safety Ball \\
 MLP &  Multilayer Perceptron\\
 EPA &  Extreme Point Anchor \\
 NLP & Nonlinear Programming  \\
 MSE &  Mean Squared Error\\
\hline
\end{tabular}
\label{tab:acronyms}
\end{table}
\section{Preliminaries}
\label{sec:Preliminaries}

In this section, we provide mathematical preliminaries on safety-critical control.  
We consider a discrete-time control system in the form:
\begin{equation}
\label{eq:discrete-dynamics}
\mathbf{x}_{t+1} = f(\mathbf{x}_t, \mathbf{u}_t),
\end{equation}
where $\mathbf{x}_t \in \mathcal{X} \subset \mathbb{R}^n$ is the state at time step $t \in \mathbb{N}$, $\mathbf{u}_t \in \mathcal{U} \subset \mathbb{R}^q$ is the control input, and the function $f: \mathbb{R}^n \times \mathbb{R}^q \to \mathbb{R}^n$ is assumed to be locally Lipschitz continuous.
We define safety as the forward invariance of a set $\mathcal{C}$. Specifically, system \eqref{eq:discrete-dynamics} is considered safe if, once initialized in $\mathcal{C}$, its state remains in $\mathcal{C}$ for all future time steps. We define $\mathcal{C}$ as the superlevel set of a continuous function $h: \mathbb{R}^n \to \mathbb{R}$:
\begin{equation}
\label{eq:safe-set}
\mathcal{C} \coloneqq \{ \mathbf{x} \in \mathbb{R}^n : h(\mathbf{x}) \geq 0 \}.
\end{equation}
\begin{definition}[Relative degree~\cite{liu2023iterative}]
\label{def:relative-degree}
The output $y_{t}=h(\mathbf{x}_{t})$ of system \eqref{eq:discrete-dynamics} is said to have relative degree $m$ if
\begin{equation}
\begin{split}
&y_{t+i}=h(\bar{f}_{i-1}(f(\mathbf{x}_{t},\mathbf{u}_{t}))), \ i \in \{1,2,\dots,m\},\\
 \text{s.t.} & \ \frac{\partial y_{t+m}}{\partial \mathbf{u}_{t}} \ne \textbf{0}_{q}, \frac{\partial y_{t+i}}{\partial \mathbf{u}_{t}}= \textbf{0}_{q},  \ i \in \{1,2,\dots,m-1\},
\end{split}
\end{equation}
i.e., $m$ is the number of steps (delay) in the output $y_{t}$ in order for any component of the control input $\mathbf{u}_{t}$ to explicitly appear ($\textbf{0}_{q}$ is the zero vector of dimension $q$). 
\end{definition}

In Def.~\ref{def:relative-degree}, $\bar{f}(\mathbf{x})\coloneqq f(\mathbf{x},0)$ denotes the uncontrolled state dynamics. 
The subscript of $\bar{f}$ indicates recursive composition:
$\bar{f}_{0}(\mathbf{x})=\mathbf{x}$ and 
$\bar{f}_{i}(\mathbf{x})=\bar{f}(\bar{f}_{i-1}(\mathbf{x}))$ for $i\ge1$.
We assume that $h(\mathbf{x})$ has
relative degree $m$ with respect to system (\ref{eq:discrete-dynamics}) based on Def. \ref{def:relative-degree}.
Starting with $\psi_{0}(\mathbf{x}_{t})\coloneqq h(\mathbf{x}_{t})$, we define a sequence of discrete-time functions $\psi_{i}:  \mathbb{R}^{n}\to\mathbb{R}$, $i=1,\dots,m$ as:
\begin{equation}
\label{eq:high-order-discrete-CBFs}
\psi_{i}(\mathbf{x}_{t})\coloneqq \bigtriangleup \psi_{i-1}(\mathbf{x}_{t})+\alpha_{i}(\psi_{i-1}(\mathbf{x}_{t})), 
\end{equation}
where $\bigtriangleup \psi_{i-1}(\mathbf{x}_{t})\coloneqq \psi_{i-1}(\mathbf{x}_{t+1})-\psi_{i-1}(\mathbf{x}_{t})$, and $\alpha_{i}(\cdot)$ denotes the $i^{th}$ class $\kappa$ function which satisfies $\alpha_{i}(\psi_{i-1}(\mathbf{x}_{t}))\le \psi_{i-1}(\mathbf{x}_{t})$ for $i=1,\ldots, m$.
A sequence of sets $\mathcal {C}_{i}$ is defined based on \eqref{eq:high-order-discrete-CBFs} as
\begin{equation}
\label{eq:high-order-safety-sets}
\mathcal {C}_{i}\coloneqq \{\mathbf{x}\in \mathbb{R}^{n}:\psi_{i}(\mathbf{x})\ge 0\}, \ i \in\{0,\ldots,m-1\}.
\end{equation}

\begin{definition}[DHOCBF~\cite{xiong2022discrete}]
\label{def:high-order-discrete-CBFs}
Let $\psi_{i}(\mathbf{x}), \ i\in \{1,\dots,m\}$ be defined by \eqref{eq:high-order-discrete-CBFs} and $\mathcal {C}_{i},\ i\in \{0,\dots,m-1\}$ be defined by \eqref{eq:high-order-safety-sets}. A function $h:\mathbb{R}^{n}\to\mathbb{R}$ is a Discrete-time High-Order Control Barrier Function (DHOCBF) with relative degree $m$ for system \eqref{eq:discrete-dynamics} if there exist $\psi_{m}(\mathbf{x})$ and $\mathcal {C}_{i}$ such that
\begin{equation}
\label{eq:highest-order-CBF}
\psi_{m}(\mathbf{x}_{t})\ge 0, \ \forall x_{t}\in \mathcal{C}_{0}\cap \dots \cap \mathcal {C}_{m-1}, t\in\mathbb{N}.
\end{equation}
\end{definition}

\begin{theorem}[Safety Guarantee \cite{xiong2022discrete}]
\label{thm:forward-invariance}
Given a DHOCBF $h(\mathbf{x})$ from Def. \ref{def:high-order-discrete-CBFs} with corresponding sets $\mathcal{C}_{0}, \dots,\mathcal {C}_{m-1}$ defined by \eqref{eq:high-order-safety-sets}, if $\mathbf{x}_{0} \in \mathcal {C}_{0}\cap \dots \cap \mathcal {C}_{m-1},$ then any Lipschitz controller $\mathbf{u}_{t}$ that satisfies the constraint in \eqref{eq:highest-order-CBF}, $\forall t\ge 0$ renders $\mathcal {C}_{0}\cap \dots \cap \mathcal {C}_{m-1}$ forward invariant for system \eqref{eq:discrete-dynamics}, $i.e., \mathbf{x}_{t} \in \mathcal {C}_{0}\cap \dots \cap \mathcal {C}_{m-1}, \forall t\ge 0.$
\end{theorem}
We can simply define an $i^{th}$ order DCBF $\psi_{i}(\mathbf{x})$ in \eqref{eq:high-order-discrete-CBFs} as
\begin{equation}
\label{eq:simple-high-order-discrete-CBFs}
\psi_{i}(\mathbf{x}_{t})\coloneqq \bigtriangleup \psi_{i-1}(\mathbf{x}_{t})+\gamma_{i}\psi_{i-1}(\mathbf{x}_{t}),
\end{equation}
where $\alpha(\cdot)$ is defined linear and $0<\gamma_{i}\le 1, i\in \{1,\dots,m\}$.
The above DCBF is typically incorporated as a safety constraint into the following safety-critical optimal control problem:

\begin{subequations}
\label{eq:classical ocp}
\begin{align}
J(\mathbf{u}_{t}, &\mathbf{x}_{t+1})=\min_{\mathbf{u}_{t}} \mathbf{u}_{t}^{T}\mathbf{u}_{t}+\sum_{k=t}^{t+1}(\mathbf{x}_{k}-\mathbf{x}_{T})^{\top}(\mathbf{x}_{k}-\mathbf{x}_{T})\label{subeq:obj}\\
s.t. \ \ 
&\psi_{m-1}(\mathbf{x}_{t+1}) \geq (1 - \gamma_{m}) \psi_{m-1}(\mathbf{x}_{t}), ~0<\gamma_{m}\le1,\label{subeq:dcbf1}\\
&\mathbf{u}_{t}\in \mathcal U \subset \mathbb{R}^{q},~\mathbf{x}_{t+1} \in \mathcal X \subset \mathbb{R}^{n}\label{subeq:control and state limits},
\end{align}
\end{subequations}
where \eqref{subeq:obj} specifies the objective function of the optimization problem, $\mathbf{x}_{T}$ denotes the target state, and $t \in \mathbb{N}$. At each time step $t$, the current state $\mathbf{x}_t$ is given. Constraint \eqref{subeq:dcbf1} enforces that the system state $\mathbf{x}_{t+1}$ in \eqref{eq:discrete-dynamics} remains within the safe set $\mathcal{C}$ defined in \eqref{eq:safe-set} (Thm. \ref{thm:forward-invariance}). Constraint \eqref{subeq:control and state limits} imposes bounds on both the control input and the state; however, these bounds may conflict with the safety constraint \eqref{subeq:dcbf1}, potentially leading to infeasibility.

\section{Problem Formulation and Approach}
\label{sec:Problem Formulation and Approach}
\begin{figure}
\vspace*{3mm}
    \centering
\includegraphics[scale=0.3]{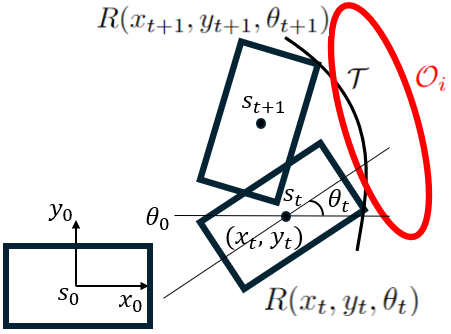}
    \caption{Rigid-body transition of the robot (black) between two discrete time steps with respect to an obstacle (red). $(x_t, y_t)$ denotes the location, and $\theta_{t}$ denotes the orientation of the robot at $t$. $\mathcal{T}$ denotes the continuous trajectory of a boundary point on the robot between time steps $t$ and $t+1$.}
    \label{fig:Robot}
\end{figure}

We consider a 2D car-like robot with arbitrary, possibly nonconvex geometry navigating through a cluttered environment. The robot must reach a specified target while ensuring safety and adhering to its dynamics and input constraints. The environment contains $K$ static obstacles $\mathcal{O}_1, \dots, \mathcal{O}_K \subset \mathbb{R}^2$, each with known, fixed boundaries. Let $\mathcal{O} \coloneqq \bigcup_{i=1}^K \mathcal{O}_i$ denote the union of all obstacles. Let $\mathcal{R}(x_t, y_t, \theta_t) \subset \mathbb{R}^2$ denote the robot at time $t$, 
where $(x_t, y_t)$ is the position of a reference point of the robot and $\theta_t$ is its orientation with respect to a reference frame (see Fig.~\ref{fig:Robot}). 
The (safety) requirement that the robot avoids all obstacles at all times can be written as:
\begin{equation}
\mathcal{R}(x_t, y_t, \theta_t) \cap \mathcal{O} = \emptyset, ~ \forall t\ge 0.
\end{equation}


We assume that the dynamics of the robot are described by \eqref{eq:discrete-dynamics}, where the system state $\mathbf{x}_t$ includes the robot’s location $(x_t, y_t)$, orientation $\theta_t$, and potentially additional variables such as speed. The control input is denoted by $\mathbf{u}_t$. The obstacles correspond to configuration-space (C-space) obstacles, i.e., the images of workspace obstacles mapped into the robot’s configuration space. \footnote{Although we focus on a 2D setting in this work, the proposed method is general and can be extended to higher-dimensional systems such as aerial drones or robotic manipulators.}

In this paper, we consider the following problem:

\begin{problem}
\label{prob:Path-prob}
For a robot with given geometry $\mathcal{R}$  and dynamics~\eqref{eq:discrete-dynamics} moving in an environment with obstacles  $\mathcal{O}_i$, $i=1,\ldots,K$, design a control policy that steers the reference point $(x_t, y_t)$ from an initial location $(x_0, y_0)$ to a desired region centered at the target location $(x_T, y_T)$ in finite time, while avoiding collisions, satisfying the control and state constraints in \eqref{subeq:control and state limits}, and minimizing its control effort.
\end{problem}


When considering Problem~\ref{prob:Path-prob}, we typically need to compute the clearance between the robot and obstacles to ensure it remains nonnegative. However, obtaining an analytical expression for the clearance is generally infeasible for robots and obstacles with complex geometries. To enable numerical computation, their continuous boundaries are approximated by sampled points, and the clearance is typically computed via exhaustive search over these samples. This process becomes computationally expensive when the boundary samples are overly dense. Furthermore, estimating the clearance for future configurations—such as $(x_{t+1}, y_{t+1}, \theta_{t+1})$—which are unknown at time step $t$, is intractable through brute-force search. In addition, enforcing nonnegative clearance only at discrete time steps, e.g., $\mathcal{R}(x_t, y_t, \theta_t) \cap \mathcal{O}_{i} = \emptyset, \mathcal{R}(x_{t+1}, y_{t+1}, \theta_{t+1}) \cap \mathcal{O}_{i} = \emptyset$ as illustrated in Fig.~\ref{fig:Robot}, does not prevent the transitional trajectory $\mathcal{T}$ of the robot between $t$ and $t+1$ from intersecting the obstacle—especially in the presence of complex shapes. To overcome these challenges, our approach is as follows.

\textbf{Approach:} We sample the boundaries of the robot and the union of obstacles using sufficiently dense boundary points. Based on these samples, we collect a large and diverse set of robot configurations along with their corresponding ground-truth clearances—computed via exhaustive search—to train a residual neural network that accurately and efficiently predicts the robot–obstacle clearance. Even when the robot's configuration is represented as an optimization variable at a future time step, the network can output a clearance prediction that serves as a decision variable to the optimization problem. This predicted clearance defines the radius of a reference ball, which is anchored to an extreme boundary point of the robot using the Extreme Point Anchor (EPA) method, forming the Local Safety Ball (LSB). The LSB ensures continuous-time collision avoidance. Its boundary is encoded as a Discrete-Time High-Order Control Barrier Function (DHOCBF) constraint with respect to the robot's reference point. Combined with state and input constraints, these DHOCBFs define a nonlinear optimization problem with the cost defined by Eq. \eqref{subeq:obj}, illustrated in Fig.~\ref{fig:ControlFrame}, whose feasibility is enhanced by a relaxation technique. The target location is included in the cost function as a reference state to minimize the deviation from the current location. This framework guarantees that the robot’s entire rigid-body motion between consecutive time steps remains collision-free. The resulting controller balances target-reaching performance with obstacle avoidance and is applied iteratively at each time step until the desired region is reached.
\begin{figure}
\vspace*{3mm}
    \centering
    \includegraphics[scale=0.26]{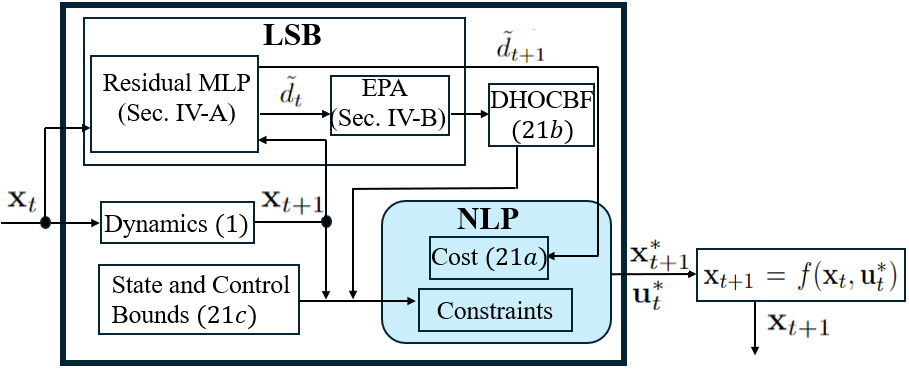}
    \caption{Schematic of the control computation at time $t$. The MLP predicts the clearance, which is passed to the EPA to define the radius of the LSB. The LSB's boundary is then converted into DHOCBF constraints. At each time step, an NLP with a quadratic cost and nonlinear constraints is solved.}
    \label{fig:ControlFrame}
\end{figure}

\section{From Learned Clearance to Safe Control: Constructing and Integrating Local Safety Balls}
\label{sec:LSBs and Control}
\subsection{Learning Clearance: Residual MLP Design}
\label{subsec: Residual MLP}
\subsubsection{Geometric Modeling}
Let $\{\mathcal{O}_k\}_{k=1}^K$ denote the workspace obstacles.
Each obstacle boundary is discretized into samples
$\partial \mathcal{O}_k=\{o_{k,m}\in\mathbb{R}^2\}_{m=1}^{M_k}$ and we define the
global boundary-sample set
$\partial \mathcal{O} \coloneqq \bigcup_{k=1}^K \partial \mathcal{O}_k$.
The robot has a nontrivial shape with boundary $\partial\mathcal{R}$.
We choose a reference point $s=[x,y]^{\top}\in\mathbb{R}^2$ whose
motion is governed by the dynamic model, and an orientation $\theta\in(-\pi,\pi]$.
For configuration $(x,y,\theta)$, the robot's workspace boundary is
\begin{equation}
\label{eq:matrixTrans}
\partial\mathcal{R}(x, y, \theta) = s \oplus \left\{ \mathbf{R}(\theta)\,r_0 \mid r_0 \in \partial\mathcal{R}_0 \right\},
\end{equation}
where $\oplus$ denotes the Minkowski sum, $\mathbf{R}(\theta)$ is a 2D rotation matrix, and $\partial\mathcal{R}_0$ denotes the robot boundary defined in the reference frame located at the origin with zero orientation. The robot–obstacle clearance $d$ is calculated as:
\begin{equation}
\label{eq:clearance}
d(x,y,\theta) =
\min_{r\in \mathcal \partial R(x,y,\theta)}\;\min_{o\in  \partial \mathcal O} \|r-o\|_2,
\end{equation}
where $r$ and $o$ are samples from the boundaries of the robot and obstacle, respectively.
If $\partial\mathcal{R}(x,y,\theta)$ intersects $\mathcal{O}$, the clearance is defined as negative, equal to the maximum Euclidean distance from the robot boundary samples inside $\mathcal{O}$ to its boundary. 
\subsubsection{Ground-Truth Clearance}
Given $(x,y,\theta)$, the clearance is defined by \eqref{eq:clearance}.
To populate the dataset, we uniformly (or according to a coverage
strategy) sample locations $s_i=[x_i,y_i]^{\top}$ across the map and discretize
orientations $\theta_j\in\Theta=\{\theta_1,\dots,\theta_{N_\theta}\}$.
For each $(x_i,y_i,\theta_j)$ we compute $d_{ij}$ from
\eqref{eq:clearance}. The resulting dataset is
$\mathcal{D}=\{(\mathbf{x}_{ij}, d_{ij})\}$ with
$\mathbf{x}_{ij}=[x_i,y_i,\theta_j]^{\top}$
.

\subsubsection{Network and Target Normalization}
We train a fully connected residual network \cite{he2016deep} $f_\phi:\mathbb{R}^3\!\to\!\mathbb{R}$ as an MLP
to regress the clearance. Inputs are standardized feature-wise to
$\bar{\mathbf{x}}=(\mathbf{x}-\mu_x)\oslash\sigma_x$ ($\oslash$ denotes elementwise division, $\mu_x,\sigma_x$ are the mean and standard deviation of the input features). To reduce target
scale imbalance, we apply a log transform:
\begin{equation}
\label{eq:target_norm}
\bar d \;=\; \frac{\log(d+\varepsilon)-\mu_{\log}}{\sigma_{\log}},
\quad \varepsilon>|d_{\text{min}}|,
\end{equation}
where $\mu_{\log},\sigma_{\log}$ denote the mean and standard deviation of 
the log-transformed distances. The training objective is mean squared error in the normalized space:
\begin{equation}
\label{eq:mse_loss}
\mathcal{L}(\phi)
= \frac{1}{|\mathcal{D}|}\sum_{(\mathbf{x},d)\in\mathcal{D}}
\big\| f_\phi(\bar{\mathbf{x}}) - \bar d \big\|_2^2 .
\end{equation}
At inference, we recover metric distances via
\begin{equation}
\label{eq:denorm}
\tilde d \;=\; \exp\!\big(\, f_\theta(\bar{\mathbf{x}})\,\sigma_{\log}+\mu_{\log}\big) - \varepsilon.
\end{equation}

\subsubsection{Architecture}
\label{sebsebsec:MLP Arch}
The network comprises three input-processing layers expanding
$3 \!\to\! h/4 \!\to\! h/2 \!\to\! h$ (with BatchNorm \cite{ioffe2015batch} and GELU \cite{hendrycks2016gaussian}), a stack of
$L_b\!=\!6$ residual blocks of width $h\!=\!2048$, and a three-layer head
reducing $h \!\to\! h/2 \!\to\! h/4 \!\to\! 1$ as shown in Fig. \ref{fig:ResNet_Design}.
Besides the standard identity skip inside each block, we add non-adjacent
skips every two blocks to promote long-range gradient flow:
if $\mathbf{z}_i$ is the output of block $i$,
\begin{equation}
\label{eq:nonadj-skip}
\mathbf{z}_{i+2} \;\leftarrow\; \mathbf{z}_{i+2} + \mathbf{A}_i\,\mathbf{z}_i,
\quad \mathbf{A}_i\in\mathbb{R}^{h\times h},\;\; i\in\{1,3\}.
\end{equation}

\subsubsection{Training Details}
Linear layers use Kaiming normal initialization \cite{he2015delving}. BatchNorm scale/bias are
initialized to $1/0$. We optimize with AdamW \cite{loshchilov2017decoupled} (base learning rate $10^{-4}$) and a
cyclic scheduler between $10^{-4}$ and $10^{-3}$ (halving amplitude each cycle).
We apply gradient clipping at $1.0$ and early stopping on validation loss.

\subsubsection{Evaluation}
On a held-out test set, we report mean squared error in the original
distance scale:
\begin{equation}
\label{eq: MSE}
\text{MSE}_{\text{test}}
= \frac{1}{N_{\theta}}\sum_{n=1}^{N_{\theta}}
\big(\tilde d(\mathbf{x}^{\text{test}}_n) - d^{\text{test}}_n\big)^2,
\end{equation}
where $N_{\theta}$ denotes the number of orientation samples for each location $(x,y)$ (e.g., $N_{\theta}=360$ when one sample is taken per degree).
\begin{remark}[\textit{DNN Prediction Errors and Certified Bounds}]
Our MLP does not provide strict a priori error bounds. For low-dimensional inputs (e.g., 2D states), techniques such as Lipschitz certificates \cite{tsuzuku2018lipschitz} could, in principle, yield worst-case guarantees. Incorporating such bounds will be pursued as future work.
\end{remark}
\begin{figure}
\vspace*{3mm}
    \centering
    \includegraphics[scale=0.5]{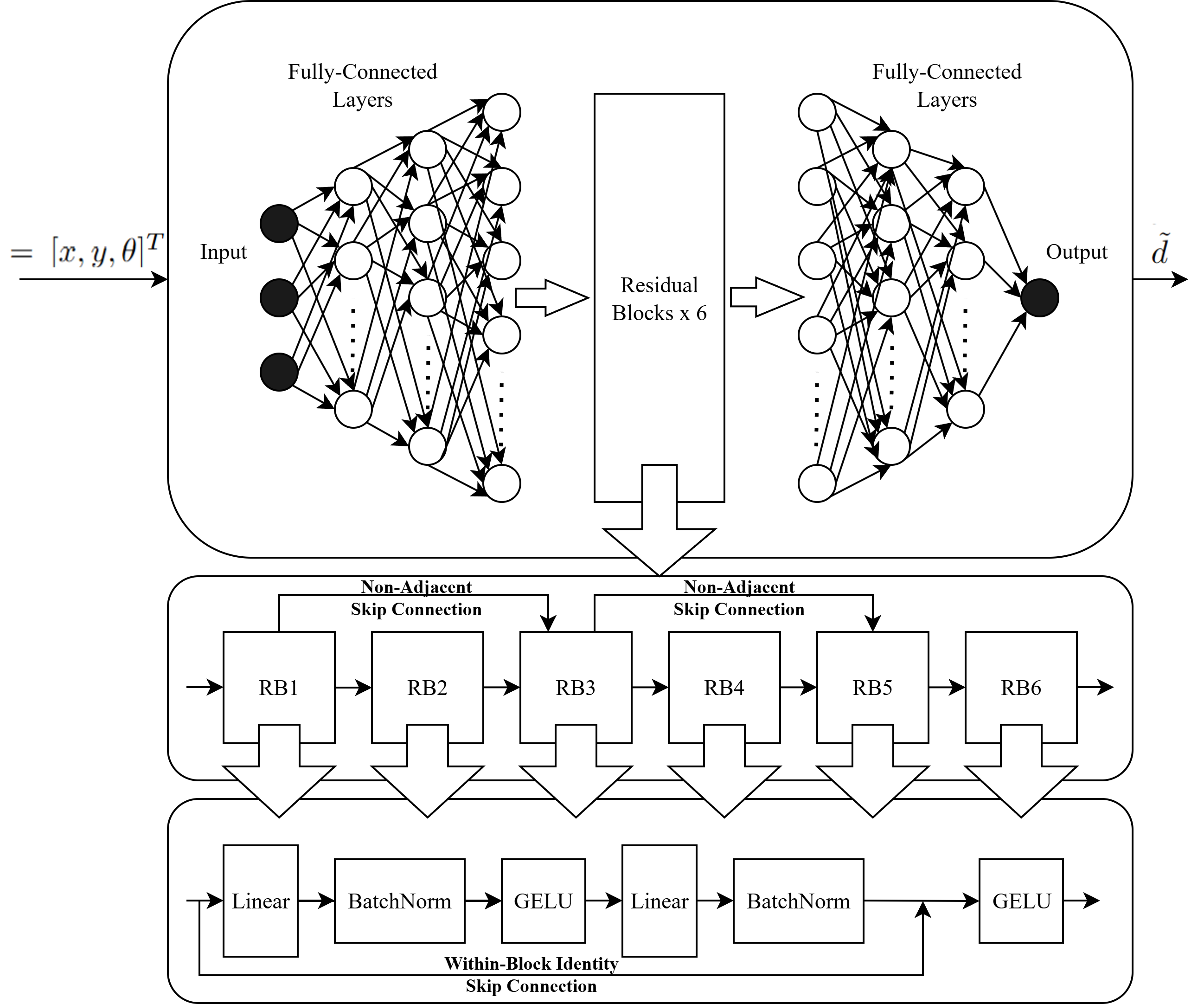}
    \caption{System identification of the residual MLP, a deep neural network $f_\phi(\mathbf{x})$ trained on map data. Given the robot’s state $(x, y, \theta)$, the MLP outputs the predicted clearance $\tilde{d}$ between the boundaries of the robot and the obstacle.}
    \label{fig:ResNet_Design}
\end{figure}

\subsection{EPA: Shape-Aware Local Safety Ball}
\label{subsec: EPA}
\begin{definition}[Extreme Point Anchor]
\label{def:EPA}
Given a robot configuration $(x,y,\theta)$, let 
$\partial\mathcal{R}(x,y,\theta)$ denote the robot boundary and 
$\partial\mathcal{O}$ the obstacle boundary. 
If the clearance $d$ is defined by \eqref{eq:clearance}
and let $(r_p,o_p)$ denote a pair of boundary points attaining the 
clearance in \eqref{eq:clearance}. 
The \emph{reference ball} is the closed ball $\mathbb{B}(r_p,d)$ centered 
at $r_p$ with radius $d$. 
Let $s$ be the reference point of the robot and define
\begin{equation}
\label{eq:center of rc}
r_c \;\in\; \arg\max_{r\in\partial\mathcal{R}(x,y,\theta)} \|r-s\|_2,
\end{equation}
i.e., the boundary point farthest from $s$ (ties are broken arbitrarily). 
The \emph{translated ball} is then given by $\mathbb{B}(r_c,d)$, i.e., the 
ball of radius $d$ centered at $r_c$. 
This construction is referred to as the Extreme Point Anchor (EPA), 
as illustrated in Fig.~\ref{fig:RobObs}. 
\end{definition}

\begin{lemma}[Ball Containment under Bounded Rigid Motion]
\label{lem:bounded move}
Let $d$ be the clearance in \eqref{eq:clearance}. Let $\mathcal{T}\in SE(2)$ be the one-step rigid motion of the robot boundary, 
and let $\{\mathcal{T}_\tau\}_{\tau\in[0,1]}$ denote a continuous transition
from the identity motion ($\tau=0$) to $\mathcal{T}$ ($\tau=1$). For each boundary point $r\in\partial\mathcal{R}$ define $r(\tau):=\mathcal{T}_\tau(r)$.
Assume the step is $d$-bounded that
\begin{equation}\label{eq:d-bounded}
\sup_{\tau\in[0,1]}\ \sup_{r\in\partial\mathcal{R}}\ \|r(\tau)-r\|_2 \;\le\; d.
\end{equation}
Then, for every $r\in\partial\mathcal{R}$ and every $\tau\in[0,1]$, we have $r(\tau)\in\mathbb{B}(r,d)$ and, in particular, $r'=\mathcal{T}(r)\in\mathbb{B}(r,d)$. Consequently, since each ball $\mathbb{B}(r,d)$ contains no obstacle boundary points by \eqref{eq:clearance}, the step is continuous-time collision-free.
\end{lemma}

\begin{proof}
Fix $r\in\partial\mathcal{R}$. By \eqref{eq:d-bounded}, for all $\tau\in[0,1]$,
$\|r(\tau)-r\|\le d$, which is exactly the statement that $r(\tau)\in\mathbb{B}(r,d)$.
Because $d$ is the clearance to the obstacle boundary, $\mathbb{B}(r,d)$ contains no obstacle boundary points for any $r$. Hence no boundary trajectory $r(\tau)$ can intersect $\partial\mathcal{O}$ during the step, and the motion is safe.
\end{proof}

In Lemma. \ref{lem:bounded move}, the double supremum in \eqref{eq:d-bounded} corresponds to the maximum location change among all boundary points over the step, which in 
practice is attained by the boundary point $r_c$ that is farthest from 
the reference point $s$ as defined in 
Def.~\ref{def:EPA}. i.e., $
\|r_c(\tau)-r_c\|_2 \;=\; 
\sup_{r\in\partial\mathcal{R}} \|r(\tau)-r\|_2$. An advantage of using the EPA is that $r_c$ can be directly 
determined from the reference point $s$ and the robot shape, without 
requiring the closest boundary pair $(r_p,o_p)$. Another advantage, formalized in Lemma~\ref{lem:bounded move} (e.g., Eq.~\eqref{eq:d-bounded}), is that by constraining $r_c(\tau)$ to remain within its translated ball, all other boundary points are guaranteed to remain within their corresponding balls throughout the entire rigid-body transition, thereby reducing the number of constraints. 

\subsection{DHOCBF-Constrained NLP Controller}
\label{subsec: NLP}
When we obtain the translated ball, we need to consider how to use it to restrict the movement of the robot’s reference point $s$, so that the robot’s trejectory remains collision-free. According to Lemma 1, we must constrain $r(\tau)$ to lie within the ball centered at $r_c$ with radius $d$. We define a DHOCBF 
\begin{equation}
\label{eq:LSB-DHOCBF1}
\tilde{\psi}_{0}(\mathbf{x}_{t}|\tilde{d}_t,r_{c,t})\coloneqq \tilde{h}(\mathbf{x}_{t}|\tilde{d}_t,r_{c,t})= \tilde{d}_t-\|r_{c}(\mathbf{x}_{t})-r_{c,t}\|_2 , 
\end{equation}
where $\tilde{d}_t=f_{\phi}(x_{t},y_{t},\theta_{t})$ is obtained at $t$ from the MLP and $r_c(\mathbf{x}_t)=r_{c,t}$ denotes the center of the translated ball at $t$ computed according to \eqref{eq:center of rc}. Since the LSB constructed based on $\tilde{d}_t$ and $r_{c,t}$ often defines a local safety region that is smaller than the actual global safety region, strictly enforcing the DHOCBF based on local safety region tends to shrink the feasible set and increases the risk of infeasibility. To address this, we introduce slack variables to relax the DHOCBF constraints while still ensuring obstacle avoidance.
\begin{theorem}
\label{thm:relaxed forward-invariance}
Given a DHOCBF $\tilde{\psi}_{0}(\mathbf{x},\Omega|\tilde{d}_t,r_{c,t})=\tilde{h}(\mathbf{x}|\tilde{d}_t,r_{c,t})$ for system \eqref{eq:discrete-dynamics} with corresponding relaxed functions $\tilde{\psi}_{i}(\mathbf{x},\Omega|\tilde{d}_t,r_{c,t}), ~i\in\{1,...,m\}$ defined by
\begin{equation}
\label{eq:new DHOCBF}
\begin{split}
\tilde{\psi}_{i}(\mathbf{x}_{t},\Omega_{t}|\tilde{d}_t,r_{c,t})\coloneqq  \tilde{\psi}_{i-1}(\mathbf{x}_{t+1},\Omega_{t}|\tilde{d}_t,r_{c,t})+ \\
\omega_{i,t}(\gamma_{i}-1)\tilde{\psi}_{i-1}(\mathbf{x}_{t},\Omega_{t}|\tilde{d}_t,r_{c,t}),
\end{split}
\end{equation}
where $0<\gamma_{i}\le 1$ and $\Omega=[\omega_{1},...,\omega_{m}]^{\top }$ is a slack variable vector. If $\tilde{\psi}_{0}(\mathbf{x}_{0},\Omega_0|\tilde{d}_0,r_{c,0})\ge 0$ and $\omega_{i}\ge0$, then any Lipschitz controller $\mathbf{u}_{t}$ that satisfies the constraints $\tilde{\psi}_{i}(\mathbf{x}_{t},\Omega_{t}|\tilde{d}_t,r_{c,t})\ge 0,~i\in\{1,...,m\}$ renders $\tilde{h}(\mathbf{x}_{t+m}|\tilde{d}_t,r_{c,t})\ge 0, \forall t\ge 0$ and obstacle avoidance is guaranteed.
\end{theorem}
\begin{proof}
At $t=0$, since $\tilde{\psi}_{0}(\mathbf{x}_{0},\Omega_0|\tilde{d}_0,r_{c,0})\ge 0$, based on \eqref{eq:new DHOCBF} and Def. \ref{def:relative-degree}, we have $\tilde{\psi}_{0}(\mathbf{x}_{1},\Omega_0|\tilde{d}_0,r_{c,0})\ge\omega_{1,0}(1-\gamma_{1})\tilde{\psi}_{0}(\mathbf{x}_{0}|\tilde{d}_0,r_{c,0})\ge0, ..., \tilde{\psi}_{m-1}(\mathbf{x}_{1},\mathbf{u}_{0},\Omega_0|\tilde{d}_0,r_{c,0})\ge\omega_{m,0}(1-\gamma_{m})\tilde{\psi}_{m-1}(\mathbf{x}_{0},\Omega_0|\tilde{d}_0,r_{c,0})\ge0$. Since 
$\tilde{\psi}_{1}(\mathbf{x}_{1},\Omega_0|\tilde{d}_0,r_{c,0})\ge 0\Leftrightarrow \tilde{\psi}_{0}(\mathbf{x}_{2},\Omega_0|\tilde{d}_0,r_{c,0})\ge\omega_{1,0}(1-\gamma_{1})\tilde{\psi}_{0}(\mathbf{x}_{1},\Omega_0|\tilde{d}_0,r_{c,0})\ge0$, repeatedly we have $\tilde{\psi}_{2}(\mathbf{x}_{1},\Omega_0|\tilde{d}_0,r_{c,0})\ge 0\Leftrightarrow \tilde{\psi}_{0}(\mathbf{x}_{3},\Omega_0|\tilde{d}_0,r_{c,0})\ge 0$, and so on up to $\tilde{\psi}_{m-1}(\mathbf{x}_{1},\mathbf{u}_{0},\Omega_0|\tilde{d}_0,r_{c,0})\ge 0\Leftrightarrow \tilde{\psi}_{0}(\mathbf{x}_{m},\Omega_0|\tilde{d}_0,r_{c,0})\ge 0$. At $t=1$, since $\tilde{\psi}_{0}(\mathbf{x}_{1},\Omega_0|\tilde{d}_0,r_{c,0})\ge 0$, we have $\tilde{\psi}_{0}(\mathbf{x}_{1},\Omega_1|\tilde{d}_1,r_{c,1})\ge 0$, based on \eqref{eq:new DHOCBF} and Def. \ref{def:relative-degree}, we have $\tilde{\psi}_{0}(\mathbf{x}_{2},\Omega_1|\tilde{d}_1,r_{c,1})\ge\omega_{1,1}(1-\gamma_{1})\tilde{\psi}_{0}(\mathbf{x}_{1}|\tilde{d}_1,r_{c,1})\ge0$, 
and so on up to $\tilde{\psi}_{m-1}(\mathbf{x}_{2},\mathbf{u}_{1},\Omega_1|\tilde{d}_1,r_{c,1})\ge\omega_{m,1}(1-\gamma_{m})\tilde{\psi}_{m-1}(\mathbf{x}_{1},\Omega_1|\tilde{d}_1,r_{c,1})\ge0$. Since 
$\tilde{\psi}_{1}(\mathbf{x}_{2},\Omega_1|\tilde{d}_1,r_{c,1})\ge 0\Leftrightarrow \tilde{\psi}_{0}(\mathbf{x}_{3},\Omega_1|\tilde{d}_1,r_{c,1})\ge\omega_{1,1}(1-\gamma_{1})\tilde{\psi}_{0}(\mathbf{x}_{2}|\tilde{d}_1,r_{c,1})\ge0$, repeatedly we have $\tilde{\psi}_{2}(\mathbf{x}_{2},\Omega_1|\tilde{d}_1,r_{c,1})\ge 0\Leftrightarrow \tilde{\psi}_{0}(\mathbf{x}_{4},\Omega_1|\tilde{d}_1,r_{c,1})\ge 0$, and so on up to $\tilde{\psi}_{m-1}(\mathbf{x}_{2},\mathbf{u}_{1},\Omega_1|\tilde{d}_1,r_{c,1})\ge 0\Leftrightarrow \tilde{\psi}_{0}(\mathbf{x}_{1+m},\Omega_1|\tilde{d}_1,r_{c,1})\ge 0$. Repearedly do above process we have $\tilde{\psi}_{i}(\mathbf{x}_{t+1},\Omega_{t}|\tilde{d}_t,r_{c,t})\ge 0,~i\in\{0,...,m-1\}$ renders $\tilde{\psi}_{0}(\mathbf{x}_{t+m},\Omega_t|\tilde{d}_t,r_{c,t})=\tilde{h}(\mathbf{x}_{t+m}|\tilde{d}_t,r_{c,t})\ge 0,$ therefore, for all $t \ge 0$, the clearance between the robot 
boundary and the obstacle boundary over the horizon $[0,\,t+m]$ remains 
nonnegative, and thus obstacle avoidance is guaranteed.
\end{proof}

In Thm.~\ref{thm:relaxed forward-invariance}, the slack variable $\omega_{i} \in [0,1)$ relaxes the DHOCBF constraint to improve feasibility, scaling with $1-\gamma_i$ to adjust the class $\kappa$ function and ensure obstacle avoidance for time-varying safety sets. Under high-order constraints, it also extends safety guarantees forward by $m$ steps, enabling an $m$-step lookahead. While DHOCBFs provide step-wise discrete-time safety, Lemma~\ref{lem:bounded move} ensures that the geometry of the LSB yields collision-free rigid-body motion over the continuous interval $\tau \in [t\Delta t, (t+m)\Delta t]$, where $\Delta t > 0$ is the sampling period. Thus, the LSB construction bridges discrete-time control and continuous-time safety. This enables us to corporate relaxed DHOCBF constraints in a nonlinear programming (NLP) problem:
\begin{subequations}
\label{eq:one step NLP}
\begin{align}
&\min_{\mathbf{u}_{t},\Omega_{t}} 
p(\mathbf{x}_{t+1})+ q(\mathbf{x}_{t},\mathbf{u}_{t},\Omega_{t})-\lambda\tilde{d}_{t+1}\label{subeq:obj3}\\
s.t. \ \ 
&\tilde{\psi}_{i-1}(\mathbf{x}_{t+1},\Omega_{t}|\tilde{d}_t,r_{c,t})+\notag\\ 
&\omega_{i,t}(\gamma_{i}-1)\tilde{\psi}_{i-1}(\mathbf{x}_{t},\Omega_{t}|\tilde{d}_t,r_{c,t})\ge 0,\label{subeq:dcbf3}\\
&\mathbf{u}_{t}\in \mathcal U \subset \mathbb{R}^{q},~\mathbf{x}_{t+1} \in \mathcal X \subset \mathbb{R}^{n}\label{subeq:control and state limits3},
\end{align}
\end{subequations}
where $0<\gamma_{i}\le 1$, $\Omega=[\omega_{1},...,\omega_{m}]^{\top }$, and the current state $\mathbf{x}_t$ is given. In \eqref{subeq:obj3}, the cost function consists of the current cost 
$q(\mathbf{x}_{t},\mathbf{u}_{t},\Omega_{t})$, the next-step cost 
$p(\mathbf{x}_{t+1})$, and a clearance penalty term 
$-\lambda \tilde{d}_{t+1}$, where $\lambda>0$ is a weight factor. The 
penalty $-\lambda \tilde{d}_{t+1} = -\lambda f_{\phi}(x_{t+1},y_{t+1},\theta_{t+1})$ 
discourages negative values of $\tilde{d}_{t+1}$ and thereby encourages the robot to maintain a greater clearance from obstacles at the next time step. The predictive MLP enables tractable clearance prediction by providing outputs that serve as decision variables in the NLP when the input configuration is defined as variables. The discrete-time dynamics in \eqref{eq:discrete-dynamics} are incorporated into the 
relaxed DHOCBF constraints in \eqref{subeq:dcbf3} to guarantee safety, e.g., $\tilde{\psi}_{0}(\mathbf{x}_{t+i},\Omega_t|\tilde{d}_t,r_{c,t})=\tilde{\psi}_{0}(\bar{f}_{i-1}(f(\mathbf{x}_{t},\mathbf{u}_{t})),\Omega_t|\tilde{d}_t,r_{c,t})$ based on Def. \ref{def:relative-degree} and $\omega_{i}\ge0, i\in\{1,...,m\}$. The state and input constraints are considered in \eqref{subeq:control and state limits3}. This NLP problem is solved at each time step, and the resulting optimal 
input $\mathbf{u}_{t}^{\ast}$ is applied to system \eqref{eq:discrete-dynamics} to generate the state 
$\mathbf{x}_{t+1}$ at the next time step.
\begin{figure}
\vspace*{3mm}
    \centering
\includegraphics[scale=0.3]{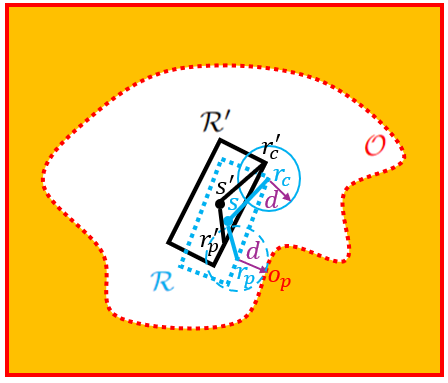}
    \caption{Illustration of the Extreme Point Anchor. The reference ball (dashed blue) is centered at boundary point $r_p$ with clearance $d$ to $o_p$, and translated to the farthest boundary point $r_c$ from the reference point $s$, forming the translated ball (solid blue).}
    \label{fig:RobObs}
\end{figure}
\begin{remark}
In Thm.~\ref{eq:new DHOCBF}, the relaxation technique places the slack variable 
$\omega_{i}$ in front of $(1-\gamma_{i})$, yielding
\begingroup
\small
\[
\tilde{\psi}_{i-1}(\mathbf{x}_{t+1},\Omega_t|\tilde{d}_t,r_{c,t}) \;\ge\;
\omega_{i,t}(1-\gamma_{i})\tilde{\psi}_{i-1}(\mathbf{x}_{t},\Omega_t|\tilde{d}_t,r_{c,t}),
\]
\endgroup
which is stronger than treating $\omega_{i}$ as a additive term,
\begingroup
\small
\[
\tilde{\psi}_{i-1}(\mathbf{x}_{t+1},\Omega_t|\tilde{d}_t,r_{c,t}) \;\ge\;
(1-\gamma_{i})\tilde{\psi}_{i-1}(\mathbf{x}_{t},\Omega_t|\tilde{d}_t,r_{c,t}) + \omega_{i,t}.
\]
\endgroup
Both improve feasibility, but only the multiplicative form preserves safety guarantees when $\omega_{i,t} \in [0,1)$ and $\tilde{\psi}_{i-1}(\mathbf{x}_t) \ge 0$. In contrast, the additive form may violate safety if $\omega_{i,t} < 0$. Prior works \cite{zeng2021enhancing,liu2023iterative} use similar strategies; however, \cite{zeng2021enhancing} is limited to relative-degree-one DCBFs, and \cite{liu2023iterative} partially relaxes the decay rate via linearization. In contrast, our method fully relaxes the decay rate without linearization, enabling general nonlinear DHOCBFs.
\end{remark}

\subsection{Complexity Scaling with State and Input Dimensions and Constraint Numbers}
\label{subsec: Complexity Analysis}
The NLP problem in \eqref{eq:one step NLP} is solved at each time step, and its complexity depends on the number of state and input variables as well as the order $m$ of the DHOCBF. For a system with state dimension $n$ and input dimension $q$, the decision vector scales as $\mathcal{O}(n+q+m)$, while the constraints include bounds and $m$ DHOCBF terms. Due to the nonlinear nature of the learned LSBs and slack relaxation, the NLP remains nonconvex, and solver time grows superlinearly with the number of variables and constraints. While the formulation can be extended to an $N$-step nonlinear model predictive control for improved safety and convergence, we focus on the single-step case to limit computation time.
\section{NUMERICAL RESULTS}
\label{sec: Numerical results}
In this section, we address an autonomous navigation problem, modeling 
the controlled robot with various shapes, including a rectangle, triangle and cross-shape (nonconvex). The proposed optimization-based control algorithm is able to generate dynamically feasible, collision-free trajectories even 
in tight maze environments, as illustrated in Fig. \ref{fig: Robot Trajectory}. Animations of these 
navigation scenarios are provided at \url{https://youtu.be/dqay1R2u5-I}.
\subsection{Numerical Setup}

\subsubsection{System Dynamics}
Consider a discrete-time unicycle model in the form
\begin{equation}
\label{eq:unicycle-model}
\begin{bmatrix} x_{t+1}{-}x_t \\ y_{t+1}{-}y_t \\ \theta_{t+1}{-}\theta_t \\ v_{t+1}{-}v_t \end{bmatrix}{=}\begin{bmatrix} v_{t} \cos(\theta_{t}) \Delta t \\ v_{t}\sin(\theta_{t}) \Delta t \\ 0 \\ 0 \end{bmatrix}{+}\begin{bmatrix} 0 & 0 \\ 0 & 0 \\ \Delta t & 0 \\ 0 & \Delta t \end{bmatrix}
\begin{bmatrix} u_{1,t} \\ u_{2,t} \end{bmatrix},
\end{equation}
where $\mathbf{x}_{t}=[x_{t},y_{t},\theta_{t},v_{t}]^{\top}$ captures the 2-D location, heading angle, and linear speed; $\mathbf{u}_{t}=[u_{1,t},u_{2,t}]^{\top}$ represents angular velocity ($u_{1}$) and linear acceleration ($u_{2}$), respectively.
The system is discretized with $\Delta t = 0.05$.
System~\eqref{eq:unicycle-model} is subject to the following state and input constraints:
\begin{equation}
\begin{split}
\label{eq:state-input-constraint}
\mathcal{X}&=\{\mathbf{x}_{t}\in \mathbb{R}^{4}: -5\cdot \mathcal{I}_{4\times1} \le \mathbf{x}_{t}\le 5\cdot \mathcal{I}_{4\times1}\},\\
\mathcal{U}&=\{\mathbf{u}_{t}\in \mathbb{R}^{2}: -10\cdot \mathcal{I}_{2\times1} \le \mathbf{u}_{t}\le 10\cdot \mathcal{I}_{2\times1}\}.
\end{split}
\end{equation}
\subsubsection{System Configuration}
The rectangle, triangle, and cross-shape robots have dimensions $(0.35{\times}0.2)$, $(0.4{\times}0.3)$, and $(0.35{\times}0.35)$, respectively, with reference points located at their center, right-angle vertex, and arm intersection. To account for reference point offsets, the initial positions are set to $[0.5,4.7]^{\top }$ (rectangle and cross-shape) and $[0.5,4.6]^{\top }$ (triangle), targeting $\mathbf{x}_{T} = [4.7,0.5]^{\top }$ or $[4.7,0.4]^{\top }$ accordingly. All robots start with an initial speed of $0.5$, but the target speed is $0$ for the rectangle robot and $0.1$ for the triangle and cross-shape robots. Heading angles follow $\theta_{0} = \theta_{T} = \text{atan2}\!\left(\tfrac{y_{T}-y_{t}}{x_{T}-x_{t}}\right)$. The robot switches trajectory once its reference point enters a waypoint-centered circle of radius $0.1$. The other reference vectors are $\mathbf{u}_{r} = [0,0]^{\top }$ and $\Omega_{r} = [1,1]^{\top }$.
\subsubsection{Map Processing} The scope of the map is $x\in[0, 5]$ and $y\in[0, 5].$ The boundary samples of the robots and obstacles are manually selected.
\label{subsubsec:MLP structure}
Our neural network follows the structure in \ref{sebsebsec:MLP Arch}. For training, we generated 3,396,600 samples by sampling 9,345 random locations across the map and rotating each over $360^\circ$ (one degree per sample). An additional 686,880 samples were collected for testing using the same procedure. The prediction accuracy of the MLP is evaluated using the MSE in \eqref{eq: MSE}.
\subsubsection{DHOCBF}
The candidate DHOCBF function $\tilde{\psi}_{0}(\mathbf{x}_{t}|\tilde{d}_t,r_{c,t})$ is defined by Eq. \eqref{eq:LSB-DHOCBF1} where the predicted clearance $\tilde{d}_t$ is defined by MLP in Sec. \ref{subsec: Residual MLP} and the center of translated ball $r_{c,t}$ is defined by Eq. \eqref{eq:center of rc}. The hyperparameters $\gamma_{1}$ and $\gamma_{2}$ are set to 0.1. 
\subsubsection{NLP Design}
The cost function of the NLP problem \eqref{eq:one step NLP} consists of current cost
$q(\mathbf{x}_{t},\mathbf{u}_{t},\Omega_{t})= ||\mathbf{x}_{t}-\mathbf{x}_{T}||_Q^2 + ||\mathbf{u}_{t}-\mathbf{u}_{r}||_R^2 +||\Omega_{t}-\Omega_{r}||_S^2$, next-step cost $p(\mathbf{x}_{t})=||\mathbf{x}_{t+1}-\mathbf{x}_{T}||_P^2$ and a clearance penalty term 
$-\lambda \tilde{d}_{t+1}$, where $Q=P=100\cdot \mathcal{I}_{4}, R= \mathcal{I}_{2}$, $S=100\cdot \mathcal{I}_{2}$ and $\lambda=1000$. 
\subsubsection{Solver Configurations and CPU Specs}
\label{subsubsec: specs}
The NLP problem is formulated in Python using CasADi \cite{andersson2019casadi} and solved with IPOPT \cite{biegler2009large} on Ubuntu 20.04, running on an AMD Ryzen 7 5800U CPU, with all inference performed on the CPU. Model training and inference are implemented in PyTorch, with training carried out on a Linux desktop equipped with an NVIDIA RTX 4090 GPU to leverage hardware acceleration.
\begin{figure*}[!t]
    \vspace*{0.2cm}
    \centering
    \begin{subfigure}[t]{0.28\linewidth}
        \centering
    \includegraphics[width=1.0\linewidth]{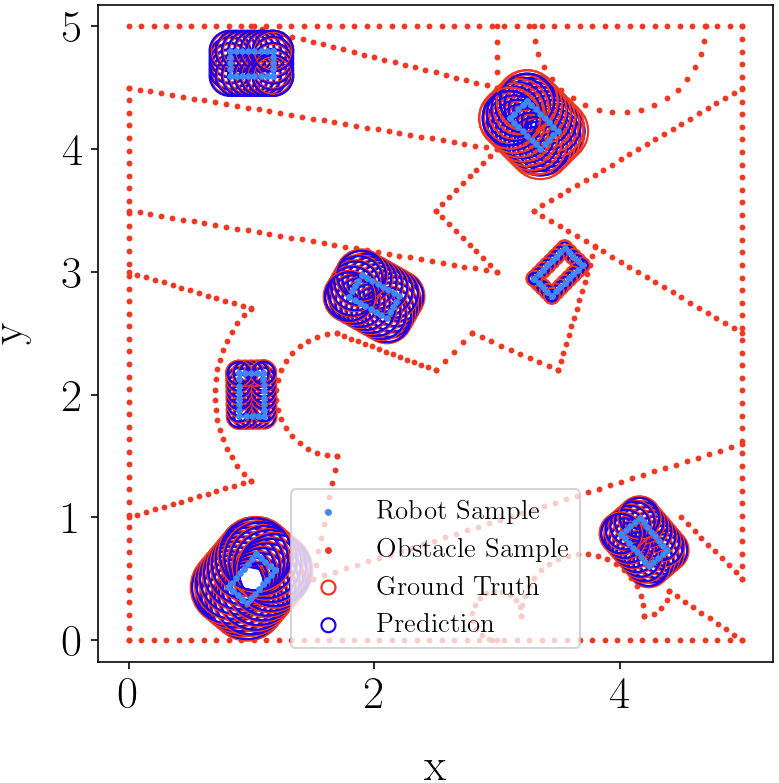}
        \caption{Rectangle robot}
        \label{subfig:1}
    \end{subfigure}
    \begin{subfigure}[t]{0.28\linewidth}
        \centering
        \includegraphics[width=1.0\linewidth]{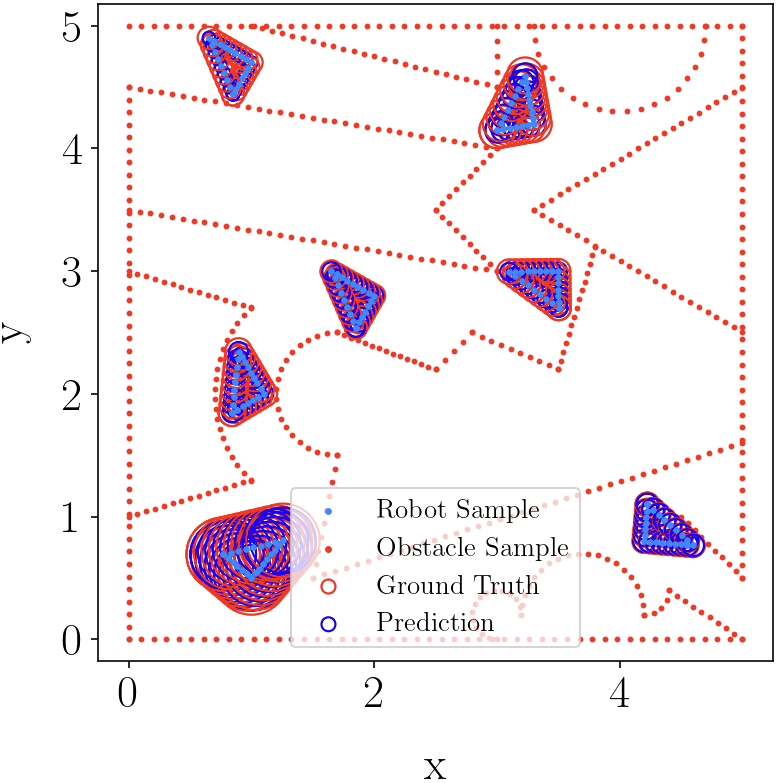}
        \caption{Triangle robot}
        \label{subfig:2}
    \end{subfigure}  
    \begin{subfigure}[t]{0.28\linewidth}
        \centering
        \includegraphics[width=1.0\linewidth]{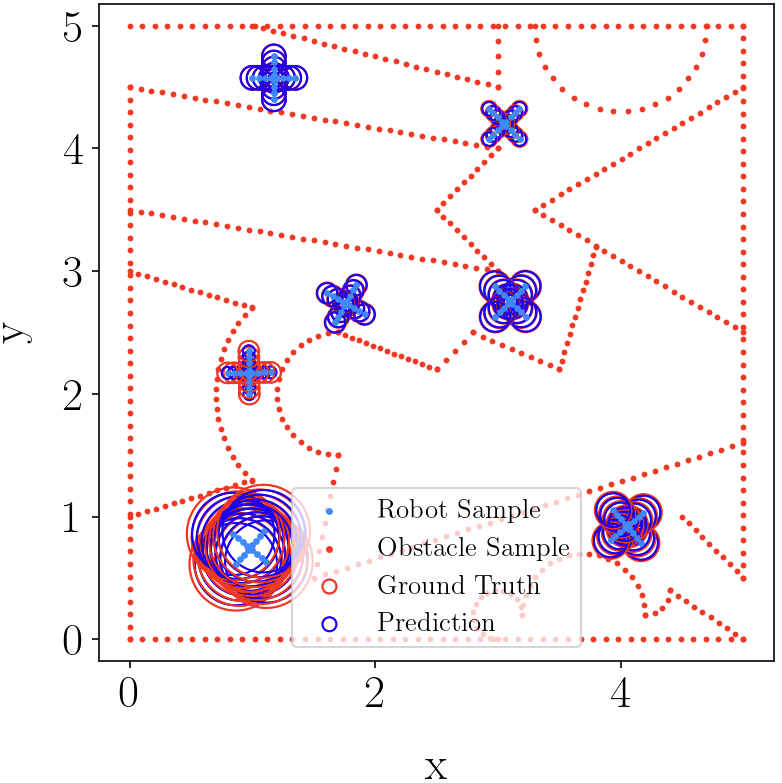}
        \caption{Cross-shape robot}
        \label{subfig:3}
    \end{subfigure}
        \begin{subfigure}[t]{0.24\linewidth}
        \centering
        \includegraphics[width=1.0\linewidth]{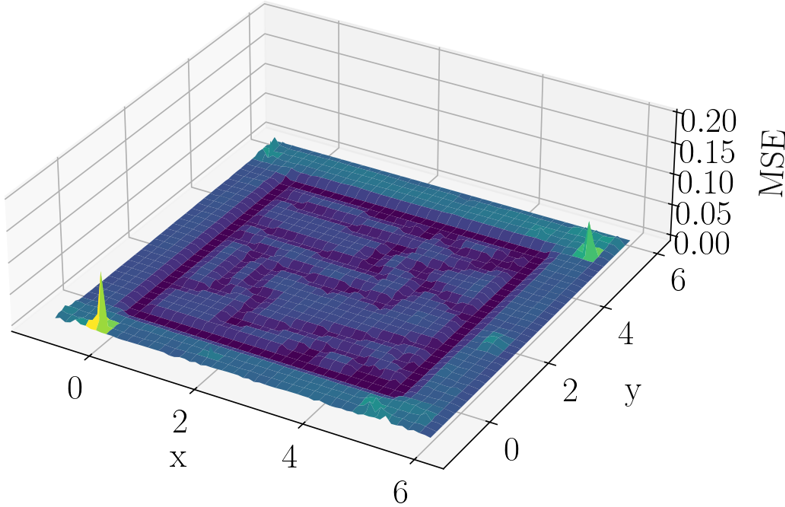}
        \caption{Isometric view}
        \label{subfig:4}
    \end{subfigure}
            \begin{subfigure}[t]{0.24\linewidth}
        \centering
        \includegraphics[width=1.0\linewidth]{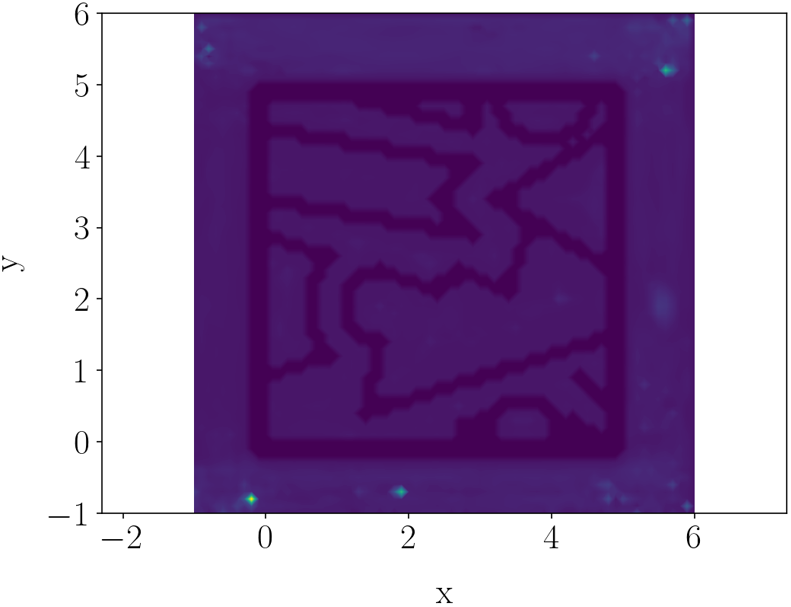}
        \caption{Top view}
        \label{subfig:5}
    \end{subfigure}
            \begin{subfigure}[t]{0.24\linewidth}
        \centering
        \includegraphics[width=1.0\linewidth]{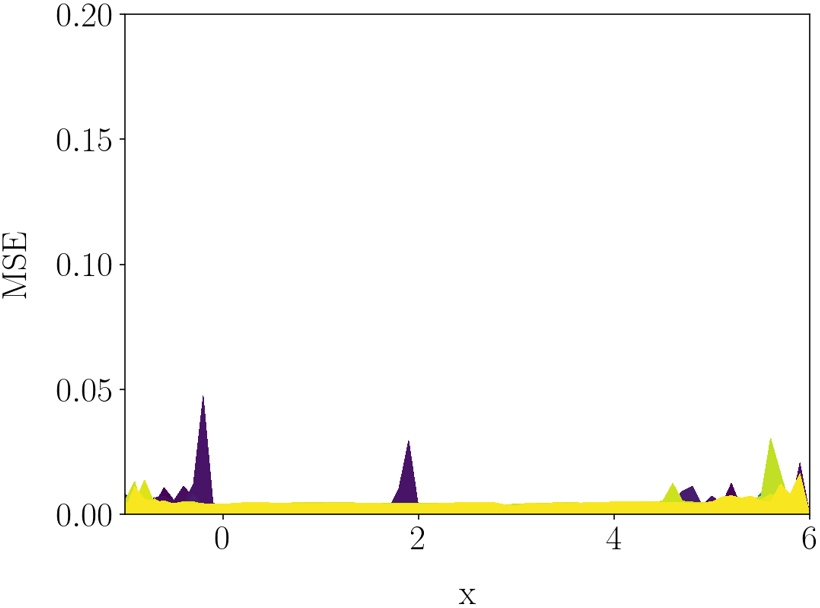}
        \caption{$x\text{-MSE}$ plane view}
        \label{subfig:6}
    \end{subfigure}
            \begin{subfigure}[t]{0.24\linewidth}
        \centering
        \includegraphics[width=1.0\linewidth]{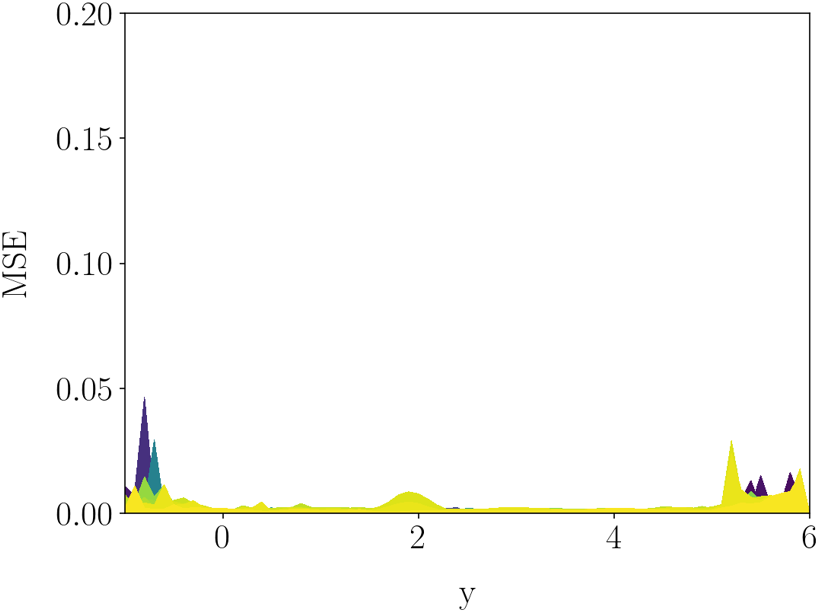}
        \caption{$y\text{-MSE}$ plane view}
        \label{subfig:7}
    \end{subfigure}
    \caption{Visualization of MLP prediction accuracy. (a–c) Comparison between ground-truth clearances (red circles) and MLP-predicted clearances (blue circles) for rectangle, triangle, and cross-shape robots, showing near coincidence of the two. (d-g) 3D visualization of the MSE of the testing data for the rectangle robot, where errors are smaller near obstacle boundaries and larger elsewhere, with most prediction errors below 0.01.}
    \label{fig: Prediction Error}
\end{figure*}
\subsubsection{MLP Prediction Accuracy}
The prediction accuracy of MLP is demonstrated in Fig. \ref{fig: Prediction Error}. We selected seven different locations from the testing data on the map, where each location corresponds to a specific configuration $(x,y,\theta)$ for each robot shape. These selected configurations are illustrated by the blue boundary samples of the robots. For each case, the ground-truth clearance was computed as the minimum distance from each robot boundary sample to the obstacle boundary samples. Using this clearance as the radius, we generated red circles centered at each sample of the robot as translated circles. Similarly, the MLP-predicted clearance was used as the radius to generate blue circles centered at each pixel. As shown in Figs.~\ref{subfig:1},~\ref{subfig:2} and~\ref{subfig:3}, the red and blue circles nearly coincide, indicating high prediction accuracy. Figs.~\ref{subfig:4},~\ref{subfig:5},~\ref{subfig:6} and \ref{subfig:7} show the relationship between the MLP’s prediction error (MSE) and the rectangle robot’s positions in 3D and projected views. The results show that the MSE dips (darker colors) near obstacle boundaries but rises (brighter colors) in other regions, implying that prediction errors are relatively larger away from obstacle boundaries. Overall, statistical analysis indicates that the prediction error at most testing locations is below $0.01$, demonstrating that the trained MLP can predict clearance with high accuracy.

\subsubsection{Safe Trajectory Generation}
Eight closed-loop trajectories connecting the waypoints are illustrated in Fig.~\ref{fig: Robot Trajectory}. For clarity, the robot locations are plotted at fixed time intervals to visualize the overall obstacle-avoidance behavior. Even within narrow passages requiring sharp turns, robots of different shapes are able to navigate safely through the free space, regardless of whether the robots or obstacles are convex or nonconvex. This demonstrates the adaptability of the proposed controller to complex environments with narrow passages and nonconvex geometries.

In addition, we record the computation time per NLP solve for each robot shape. The average single-step computation times are $14.16 \pm 1.26$ ms for the rectangle robot, $13.96 \pm 1.94$ ms for the triangle robot, and $15.82 \pm 2.12$ ms for the cross-shape robot, demonstrating the high efficiency of the optimization algorithm.

\begin{figure*}[!t]
    \vspace*{0.2cm}
    \centering
    \begin{subfigure}[t]{0.28\linewidth}
        \centering
        \includegraphics[width=1.0\linewidth]{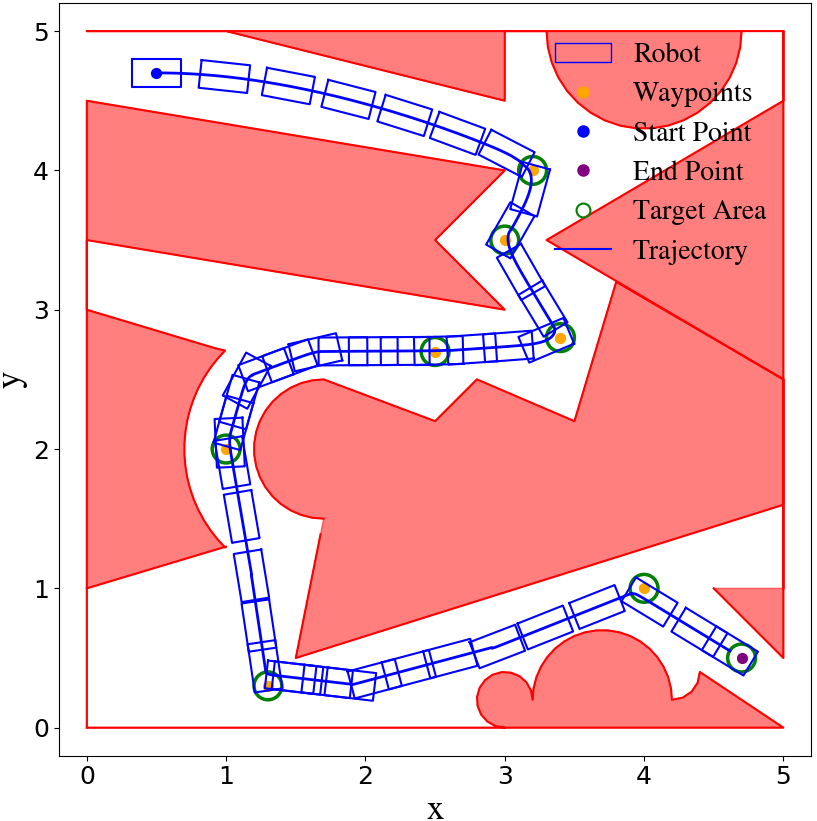}
        \caption{Rectangle robot}
        \label{subfig:8}
    \end{subfigure}
    \begin{subfigure}[t]{0.28\linewidth}
        \centering
        \includegraphics[width=1.0\linewidth]{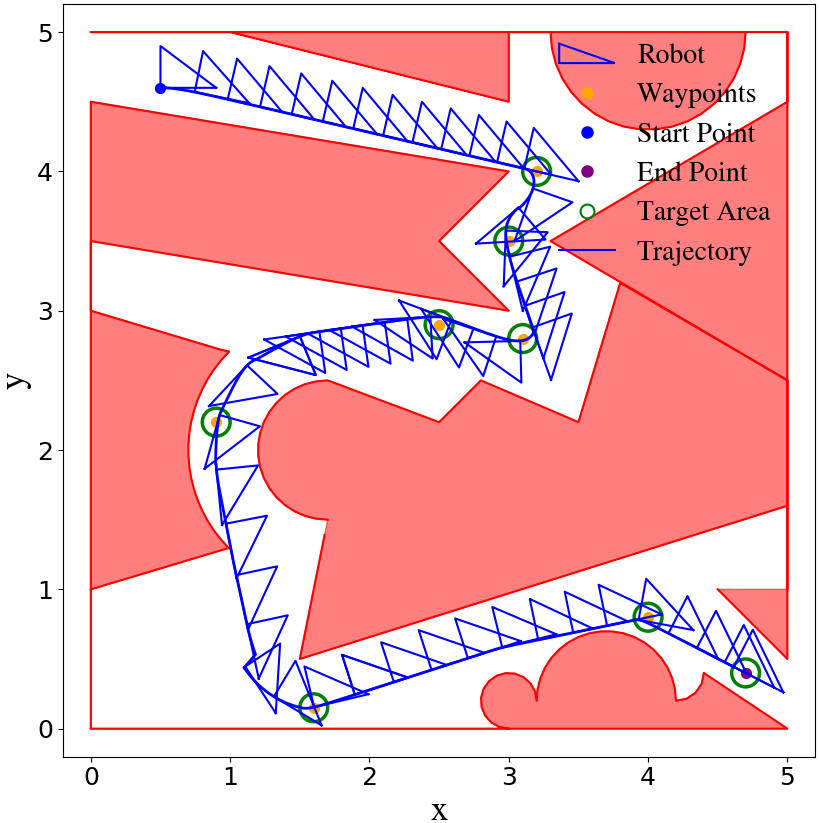}
        \caption{Triangle robot}
        \label{subfig:9}
    \end{subfigure}  
    \begin{subfigure}[t]{0.28\linewidth}
        \centering
        \includegraphics[width=1.0\linewidth]{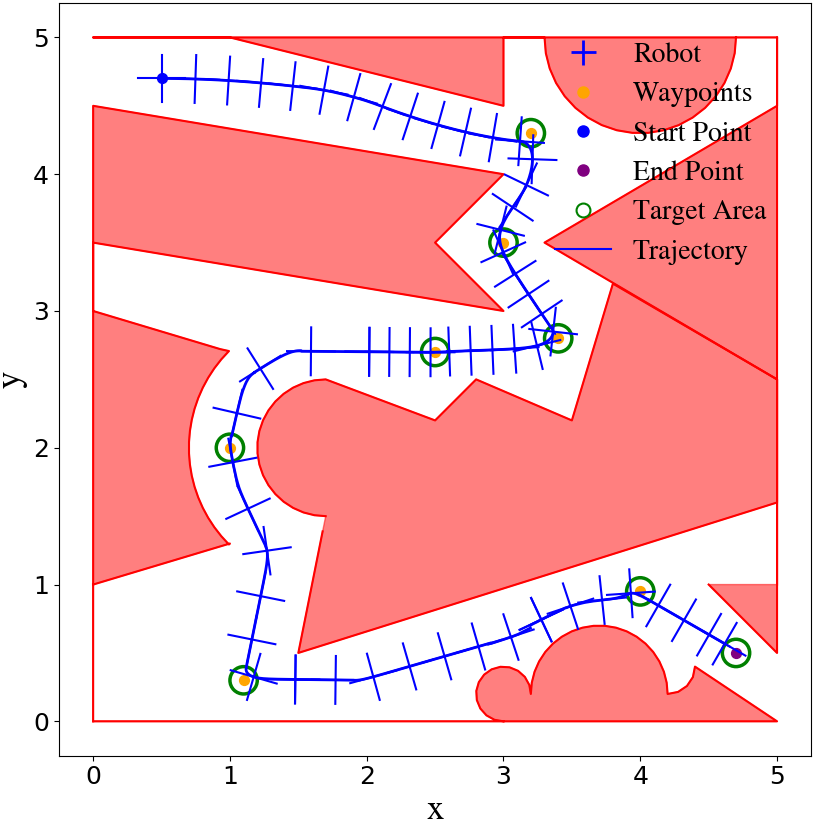}
        \caption{Cross-shape robot}
        \label{subfig:10}
    \end{subfigure}
    \caption{Closed-loop trajectories of rectangle, triangle, and cross-shape robots in a maze-like environment. Robots successfully navigate narrow passages and sharp turns, demonstrating that the proposed DHOCBF-based controller ensures safe obstacle avoidance for both convex and nonconvex robots and obstacles.}
    \label{fig: Robot Trajectory}
\end{figure*}

\section{Conclusion and Future Work}
\label{sec:Conclusion and Future Work}
In this paper, we develope a novel framework that combines learning-based clearance prediction with DHOCBF-constrained control through the construction of Local Safety Balls. By leveraging a residual MLP to predict clearances and anchoring them via the EPA, we provide a tractable way to guarantee transitional safety for robots with arbitrary shapes navigating in cluttered environments. The proposed relaxation technique further enhance feasibility for high-order constraints, and the closed-loop simulations validated both the safety and efficiency of our approach. Future research will extend this framework in several directions, such as integrating the approach into a model predictive control framework to improve long-horizon safety and performance, and considering dynamic obstacles to handle more realistic and time-varying environments.

\bibliographystyle{IEEEtran}
\balance
\bibliography{references.bib}

\end{document}